\documentclass{article}

\usepackage{placeins}

\usepackage[bottom=4cm, right=3cm, left=3cm, top=4cm]{geometry}
\usepackage{graphicx} 
\usepackage{amsmath, amssymb}
\usepackage{adjustbox}
\usepackage{multicol, multirow}
\usepackage{tabularx}
\usepackage{booktabs}
\usepackage[table]{xcolor}
\usepackage{mdframed}
\usepackage{caption}
\usepackage{float}
\usepackage{svg}
\usepackage{makecell}
\usepackage{natbib}
\usepackage{subcaption}

\usepackage[symbol]{footmisc}

\usepackage{hyperref}





\usepackage{amsthm}
\theoremstyle{definition}
\newtheorem{definition}{Definition}[section]
\newtheorem{remark}{Remark}[section]
\newtheorem{example}{Example}[section]
\newtheorem{proposition}{Proposition}[section]

\surroundwithmdframed[
  backgroundcolor=blue!5,
  linecolor=blue!50,
  skipabove=6pt,
  skipbelow=6pt,
  footnoteinside=true
]{definition}

\surroundwithmdframed[
  backgroundcolor=green!5,
  linecolor=green!50,
  skipabove=6pt,
  skipbelow=6pt,
  footnoteinside=true
]{remark}

\surroundwithmdframed[
  backgroundcolor=orange!5,
  linecolor=orange!50,
  skipabove=6pt,
  skipbelow=6pt,
  footnoteinside=true
]{example}

\surroundwithmdframed[
  backgroundcolor=purple!5,
  linecolor=purple!50,
  skipabove=6pt,
  skipbelow=6pt,
  footnoteinside=true
]{proposition}

\newcommand{\EFA}[1]{\mathcal{E}_{\mathcal{F},\mathcal{A}}(#1)}

\definecolor{real}{HTML}{A7E0A7}
\definecolor{stated}{HTML}{FFE699}    
\definecolor{estimated}{HTML}{F4A6A6}

\definecolor{fullaccess}{RGB}{0,70,130}
\definecolor{partialaccess}{RGB}{120,170,220}
\definecolor{explainonly}{RGB}{200,225,245}

\title{\vspace{-3em}On the Definition and Detection of Cherry-Picking in Counterfactual Explanations}
\author{
\begin{tabular}{c}
James Hinns\textsuperscript{1} \quad
Sofie Goethals\textsuperscript{1} \quad
Stephan Van der Veeken\textsuperscript{1} \\
Theodoros Evgeniou\textsuperscript{2} \quad
David Martens\textsuperscript{1} \\
[0.8em]
\textsuperscript{1}University of Antwerp \\
\textsuperscript{2}INSEAD
\end{tabular}
}
\date{}

\begin{document}

\maketitle

\begin{abstract}
    \noindent Counterfactual explanations are widely used to communicate how inputs must change for a model to alter its prediction. For a single instance, many valid counterfactuals can exist, which leaves open the possibility for an explanation provider to cherry-pick explanations that better suit a narrative of their choice, highlighting favourable behaviour and withholding examples that reveal problematic behaviour. We formally \textit{define} cherry-picking for counterfactual explanations in terms of an admissible explanation space, specified by the generation procedure, and a utility function. We then study to what extent an external auditor can \textit{detect} such manipulation. Considering three levels of access to the explanation process: full procedural access, partial procedural access, and explanation-only access, we show that detection is extremely limited in practice. Even with full procedural access, cherry-picked explanations can remain difficult to distinguish from non cherry-picked explanations, because the multiplicity of valid counterfactuals and flexibility in the explanation specification provide sufficient degrees of freedom to mask deliberate selection. Empirically, we demonstrate that this variability often exceeds the effect of cherry-picking on standard counterfactual quality metrics such as proximity, plausibility, and sparsity, making cherry-picked explanations statistically indistinguishable from baseline explanations. We argue that safeguards should therefore prioritise reproducibility, standardisation, and procedural constraints over post-hoc detection, and we provide recommendations for algorithm developers, explanation providers, and auditors. \vspace{-1em}
\end{abstract}

\section{Introduction}

Artificial intelligence (AI) increasingly shapes decisions that affect individuals, organisations, 
and society at large. Yet the complexity of modern AI systems often renders them difficult to 
interpret, creating a substantial obstacle to their adoption and effective 
oversight~\cite{arrieta2020explainable,samek2019explainable}. In response, the field of 
Explainable Artificial Intelligence (XAI) has emerged, developing methods that 
make the decision-making processes of AI models more 
transparent~\cite{adadi2018peeking,xu2019explainable}.
Explanations serve a variety of purposes, including improving user trust, supporting accountability, enabling compliance with regulatory requirements, while also benefiting developer understanding and model robustness. For example, the EU General Data Protection
Regulation (GDPR) is often discussed in relation to transparency and explanation obligations for individuals affected by 
algorithmic decision-making \cite{ponce2025right}, while the EU AI Act mandates explainability in high-stakes 
applications~\citep{goodman2017european}. Such requirements motivate the use of explanation methods not only for end users but also for developers and external auditors tasked with assessing compliance and potential harms.
Because there is no consensus on how to assess explanation quality, explanation methods can differ greatly from one another~\cite{de2021framework}. Moreover, many counterfactual explanation procedures rely on heuristic, non-exhaustive search rather than fully exploring the space of valid counterfactuals, so small changes to parameters or random seeds can also produce substantially different results~\cite{brughmans2024nice,glime}. This phenomenon is known as \textit{the disagreement problem}~\citep{brughmans2024disagreement,krishna2024disagreement,neely2021order,roy2022don,martens2025beware}. This plurality of explanations opens the door to \textit{cherry-picking}, where an explanation provider selectively discloses an explanation while withholding alternatives that would be preferable under an agreed criterion. While cherry-picking is frequently raised as a conceptual concern, the literature provides no precise definition of what it means to cherry-pick counterfactual explanations, and offers only limited discussion of when such behaviour can be detected or how it might be mitigated in practice.

\begin{example}
    Consider a fictional case of cherry-picking counterfactual explanations. Bob uses a model for loan decisions based on three features: income, gender, and years worked. After being rejected, Alice exercised her right under the GDPR to request an explanation. Bob generated several counterfactual explanations and presented Alice with the following:
    \begin{quote}
        ``If you had earned €10000 more and worked two more years, you would have been accepted.'' 
    \end{quote}
    This appeared reasonable to Alice, who received a raise and was accepted when she reapplied two years later. 
    However, another valid counterfactual also existed:
    \begin{quote}
        ``If you had identified as male instead, you would have been accepted.''
    \end{quote}
    Bob stated a preference for explanations with fewer feature changes, under which the second explanation was preferable, yet he presented the first to avoid revealing that the decision depended on gender. This selective presentation against his stated preference constitutes cherry-picking.
\end{example}

Our goal is twofold. First, we formalise cherry-picking for counterfactual explanations. Second, we study when, and under what information access, cherry-picking can be detected by an external auditor. To isolate the detection problem, we consider a simple setting in which an explanation provider withholds counterfactuals that would reveal reliance on sensitive attributes. We then evaluate to what extent cherry-picking can be uncovered across different levels of access to the explanation procedure. \\

Our findings show that even in this minimal setup, detection is challenging. This has direct implications for auditing and compliance: with many admissible explanations and limited audit coverage, reproducing a reported explanation does not certify the stated procedure or rule out selective disclosure. We conclude by discussing how auditors, explanation providers, and algorithm developers can use these insights to design more robust transparency mechanisms, strengthening explainability for auditing and accountability while reducing the scope for obfuscation.

\vspace{-1em}
\section{Related work}
\label{sec:related_work}

The field of XAI aims to make the predictions of machine learning models understandable to humans,
which is increasingly important in high-stakes domains where trust, accountability, and regulatory
compliance are critical \cite{doshi2017towards}. Explanations can serve multiple purposes: they can foster trust among users and decision-makers, support model developers in debugging and improving systems, reveal hidden biases or fairness concerns, and help organisations meet regulatory transparency requirements~\citep{ferrario2022explainability, goethals2024precof,martens2014explaining,vermeire2022explainable,wachter2017counterfactual}.
Test set performance alone provides a limited view of model
quality, since distinct models can achieve similar results while inducing different inductive biases
\cite{d2022underspecification,fisher2019all}. This work focuses on explanation multiplicity, where
even for a fixed model and a single data instance, multiple distinct explanations may
be returned, not only across methods but also within the same method
\cite{hinns2021initial,brughmans2024disagreement}.

\subsection{Counterfactual explanations and disagreement}

We focus on counterfactual explanations, which describe how an input instance can be minimally
changed to alter a model’s prediction \cite{martens2014explaining,wachter2017counterfactual}.
Counterfactual quality is commonly assessed using criteria such as \textit{proximity}, the distance
between the factual instance and its counterfactual \cite{brughmans2024nice,mothilal2020explaining,martens2014explaining,wachter2017counterfactual}; \textit{plausibility}, the extent to which a
counterfactual resembles realistic instances (often approximated by conformity to the empirical data
distribution); and \textit{sparsity}, the number of features changed (often seen as a special case of
proximity \cite{brughmans2024disagreement}). Beyond these, many other metrics have been proposed,
and each may be defined or implemented in several ways. Consequently, the notion of a ``minimal
change'' varies across methods, leading to multiple valid counterfactuals for the same instance. This
phenomenon is known as the \textit{disagreement problem} \cite{brughmans2024disagreement}.
Empirical studies show that counterfactual methods can produce widely different results, particularly
when they optimise for different quality metrics \cite{brughmans2024disagreement}. While we focus on
counterfactual explanations, the lack of a single agreed target is not unique to counterfactuals. What
counts as a good explanation depends on the chosen definition and implementation, and disagreement
has been observed beyond counterfactuals \cite{krishna2024disagreement,neely2021order}.

\subsection{Manipulation and fairwashing through explanations}

The possibility of returning multiple valid counterfactual explanations creates opportunities for manipulation, a risk previously identified in work on the disagreement problem \cite{barocas2020hidden,mothilal2020explaining,goethals2023manipulation}. Empirical studies show that, when multiple counterfactual explanations are available, a user can often select one that includes or omits particular features, making it easy to avoid sensitive attributes and thereby justify a decision misleadingly \cite{brughmans2024disagreement}. In settings where the interests of explanation providers diverge from those of recipients, such explanation variability can be exploited to meet the provider’s goals. \\

A prominent instance of this risk is fairwashing, where explanations are manipulated to make a model appear fairer than it truly is. Much of the research on fairwashing focuses on surrogate models: smaller, inherently interpretable models designed to approximate the behaviour of a complex black-box model. These surrogates are typically optimised for fidelity, a metric that measures how closely the surrogate reproduces the predictions of the original model on some test data. When constructing a surrogate, it is possible to optimise simultaneously for fidelity and fairness, producing a model that maintains high fidelity while appearing fairer than the black-box~\cite{aivodji2019fairwashing,lakkaraju2020fool}. Because of predictive multiplicity, there are often many surrogate models with comparable fidelity but different levels of fairness. To assess this risk, it has been proposed to examine the Rashomon set and measure the range of fairness among surrogate models with similar fidelity \cite{aivodji2021characterizing}. Another approach compares the surrogate’s fidelity to that of the black-box across subpopulations~\cite{shahin2022washing}; large subgroup disparities in fidelity can suggest fairwashing, although some divergence is inevitable whenever the surrogate is imperfect. In a related vein, it has been shown that models can behave in a biased manner on real data while producing LIME and SHAP explanations that appear unbiased, exploiting the fact that perturbations generated for these explanations often lie outside the data manifold \cite{slack2020fooling}. While several detection strategies have been proposed, fairwashing cannot be completely eliminated in general~\cite{shahin2022washing}.

\subsection{Regulation and Accountability}

To address these risks, researchers have called for stronger regulation that demands
transparency about how explanations are produced, encourages the use of white-box models,
and develops clear standards and auditing frameworks to prevent explanations from being
tailoured entirely at the provider's discretion \cite{goethals2023manipulation}.
Similar concerns have been raised in broader discussions of algorithmic accountability,
which argue for legal mechanisms that guarantee a ``right to explanation'' and ensure
regulatory oversight of automated decision-making \cite{goodman2017european}.
Others have emphasised that such oversight must extend beyond documentation to include
systematic auditing practices \cite{raji2020closing}, while promoting the use of inherently
interpretable models to avoid the opacity and manipulability of black-box systems
\cite{rudin2019stop}.

\section{Defining Cherry-Picked Explanations}

\vspace{-1em}

\begin{figure}[h]
    \centering
    \begin{minipage}[c]{0.50\textwidth} 
        \centering
        \includegraphics[width=\linewidth]{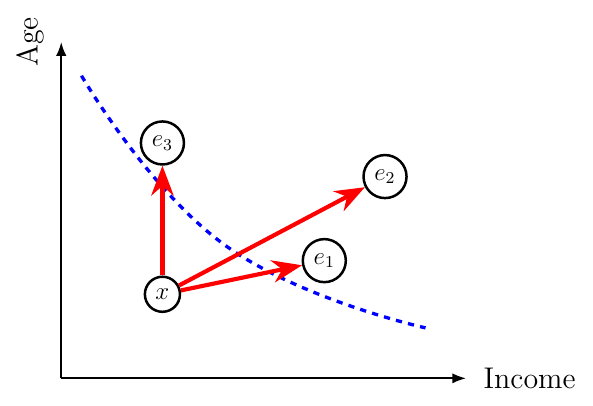}%
    \end{minipage}
    \hfill
    \begin{minipage}[c]{0.45\textwidth}
        \centering
        \small 
        \renewcommand{\arraystretch}{1.2} 
        \begin{tabular}{lcc}
            \toprule
            \textbf{ID} & \textbf{Income} & \textbf{Age} \\
            \midrule
            $x$ & €25000 & 25 \\
            \midrule
            $e_1$ & €65000 & 35 \\
            $\mathrm{cfe}_1$ & +€40000 & +10 \\
            \addlinespace
            $e_2$ & €80000 & 60 \\
            $\mathrm{cfe}_2$ & +€55000 & +35 \\
            \addlinespace
            $e_3$ & €25000 & 70 \\
            $\mathrm{cfe}_3$ & +€0 & +45 \\
            \bottomrule
        \end{tabular}
    \end{minipage}

    \caption{Toy example of multiple counterfactual instances ($e_1, e_2, e_3$) for the same data instance $x$. The corresponding counterfactual explanations are the red arrows $\mathrm{cfe}_i = e_i - x$, each crossing the blue decision boundary. The table reports the feature values for $x$ and each $e_i$, and the feature changes encoded by each $\mathrm{cfe}_i$.}
    \label{fig:cf_explanation_combined}
\end{figure}

A counterfactual identifies the minimal change to an instance such that, if that evidence were altered or removed, the model’s prediction would change \cite{martens2014explaining}. In Figure \ref{fig:cf_explanation_combined}, $x$ is the instance to be explained and $e_i$ are counterfactual instances; the counterfactual explanation is the vector $\mathrm{cfe}_i = e_i - x$, representing the minimal set of feature changes needed to change the prediction. Since a counterfactual explanation is only meaningful in relation to the instance $x$, whenever we refer to a counterfactual explanation we assume access to $x$, and by extension to the corresponding counterfactual instance $e_i$. As discussed in Section~\ref{sec:related_work}, the notion of minimality is operationalised via counterfactual quality metrics, which impose a (weak) ordering over feasible counterfactuals. In Figure \ref{fig:cf_explanation_combined}, $e_3$ is optimal under sparsity (it changes the fewest features), whereas $e_1$ is optimal under proximity when distances are comparably scaled (e.g., Euclidean distance after min–max feature scaling, or HEOM \cite{wilson1997improved}). \\

The following definition formalises a normative view of cherry-picking, assuming that the explanation specification is sound and fixed independently of the explanation outcome. We refer to the complete procedure used to generate and select counterfactuals as the explanation specification (e.g., the model(s) and data, the counterfactual generation method(s), and the ranking criterion used to order candidate counterfactuals). What constitutes a sound specification is task-specific; we give general considerations for specification soundness, but do not address questions such as fairness of the specification. This definition formalises the idea that a cherry-picked explanation is one that is not top-ranked under the stated ranking, rather than providing a directly testable detection criterion; we discuss detectability of cherry-picking in counterfactuals in Section \ref{sec:detecting}.

\begin{definition}
\label{def:cherry_picking_non_inject}
\textbf{Cherry-picked counterfactual explanation with respect to utility.} \\

Let
\begin{align*}
  &\mathcal{X} &&\text{be the instance space},\\
  &\mathcal{E} &&\text{be the counterfactual space},\\
  &\mathcal{F} &&\text{be the set of trained models considered},\\
  &\mathcal{A} &&\text{be the set of explanation methods considered},\\
  &u:\mathcal{X}\times\mathcal{E}\to\mathbb{R}
  &&\text{be the utility function}.
\end{align*}

For an instance $x \in \mathcal{X}$, define the admissible explanation space for $\mathcal{F}$ \textit{and} $\mathcal{A}$ as
\[
  \EFA{x}
  \;=\;
  \{\, A_f(x) \;\mid\; f \in \mathcal{F},\ A \in \mathcal{A} \,\}
  \subseteq \mathcal{E},
\]
where $A_f(x)$ is the explanation produced by $A$ for $f$ at $x$.
Throughout, we assume that for every \(f\in\mathcal{F}\) and \(A\in\mathcal{A}\),
the map \(A_f:\mathcal{X}\to\mathcal{E}\) and the function \(u\) are single-valued\footnotemark[1],
so that for each \(x\in\mathcal{X}\) the output \(A_f(x)\) is well-defined as a
single explanation, and repeated evaluations of \(A_f(x)\) return the same value.
\footnotetext[1]{If a model \(f\) or explanation method \(A\) is implemented as a non-deterministic procedure, we conceptually treat each fixed realisation (e.g., one choice of random seed) as a separate element of \(\mathcal{F}\) or \(\mathcal{A}\), so that \(f(x)\) and \(A_f(x)\) are well-defined.}

We assume \(\EFA{x}\) is countable and let \(\mathrm{rank}_x\) be a bijection from \(\EFA{x}\) to \(\mathbb{N}\) that induces a strict ranking of \(\EFA{x}\) consistent with \(u(x,\cdot)\), and therefore breaks ties when \(u(x,e)=u(x,e')\).

An explanation $e \in \EFA{x}$ is said to be \textit{cherry-picked} if
\[
  \mathrm{rank}_x(e) > 1.\text{\footnotemark[2]}
\]
\footnotetext[2]{Equivalently, $e$ is cherry-picked if it is not top-ranked in $\EFA{x}$, i.e., if there exists an explanation $e' \in \EFA{x}$ such that $\mathrm{rank}_x(e') < \mathrm{rank}_x(e)$.}

\end{definition}

To illustrate Definition~\ref{def:cherry_picking_non_inject}, we provide Example~\ref{ex:spars-def}, a simple toy counterfactual set in which an explanation is selected to avoid changing a sensitive feature, despite there existing a strictly better explanation under the utility \(u\) (in this case, a sparsity-based utility). The example complements the definition by giving concrete instantiations of both \(u\) and a bijective ranking \(\mathrm{rank}_x\) that is consistent with \(u\) and resolves utility ties via a fixed tie-breaking rule (here, \(\mathrm{ID}(\cdot)\)), so that cherry-picking can be verified directly from \(\EFA{x}\). This links back to the counterfactual notion of a \textit{minimal} change: in practice, minimality is determined by a utility function, so reporting any counterfactual that is not minimal among the feasible candidates under \(u\) is cherry-picking. \\

In practice, $u(x,\cdot)$ is often non-injective, so multiple explanations in $\EFA{x}$ may share the same utility.
To ensure that $\mathrm{rank}_x$ is a bijection (and hence that “the best” explanation is well-defined), the specification must include a deterministic tie-breaking rule.
A simple choice is \textit{first-found} tie-breaking, where $\mathrm{rank}_x$ returns the first explanation among those with maximal utility.
In later examples, we implement this by ordering tied explanations using a fixed identifier $\mathrm{ID}(\cdot)$.
Any such tie-breaking is arbitrary, so ties by themselves do not imply cherry-picking; however, frequent ties indicate that the utility and ranking specification is insufficiently discriminative, casting doubt on the specification’s soundness.
We return to specification soundness in Section~\ref{sec:discussion}.

\begin{example}
\textbf{Cherry-picked counterfactual with respect to sparsity utility.}
\label{ex:spars-def}

Let \(f\) be a binary loan approval model with four features:
\(\text{Income} \in \mathbb{N}\) (in €), \(\text{Gender} \in \{F,M\}\),
\(\text{Employment} \in \{\text{Temporary},\text{Permanent}\}\), and \(\text{Age} \in \mathbb{N}\) (years). \\

Consider the sparsity-based utility function:
\[
  u(x,e)
  \;=\;
  -\,\bigl\lVert x-e \bigr\rVert_0
  \;=\;
  -\sum_{j \in \{\mathrm{inc,gen,emp,age}\}} \mathbf{1}[x_j \neq e_j].
\]
Higher utility means fewer changed features (lower sparsity). Since \(u(x,\cdot)\) can tie, we define a
utility-consistent bijective rank by breaking ties using \(\mathrm{ID}(\cdot)\). \\

For this example, \(\EFA{x}\) is finite. Define \(\mathrm{rank}_x\) as the position induced by sorting
explanations in decreasing order of \(u(x,\cdot)\), with ties broken deterministically using
\(\mathrm{ID}(\cdot)\). Formally,
\[
  \mathrm{rank}_x(e) < \mathrm{rank}_x(e')
  \quad\Longleftrightarrow\quad
  \bigl(u(x,e) > u(x,e')\bigr)
  \ \ \text{or}\ \ 
  \bigl(u(x,e)=u(x,e') \ \wedge\ \mathrm{ID}(e) < \mathrm{ID}(e')\bigr).
\]
This \(\mathrm{rank}_x\) is a bijection from \(\EFA{x}\) to \(\{1,\dots,n\}\subset\mathbb{N}\),
and it satisfies \(u(x,e')>u(x,e)\Rightarrow \mathrm{rank}_x(e')<\mathrm{rank}_x(e)\).
When \(u(x,e)=u(x,e')\), the lower \(\mathrm{ID}\) receives the lower rank.

Take the instance:
\[
  x = (\text{Income}=25000,\ \text{Gender}=F,\ \text{Employment}=\text{Temporary},\ \text{Age}=30),
  \quad f(x)=0.
\]
Suppose \(\EFA{x}\) contains the following valid counterfactual explanations \(e_1,\dots,e_5\),
each satisfying \(f(e_i)=1\):
\[
\begin{array}{c|c|c|c|c|c|c}
\text{ID} & \text{Income} & \text{Gender} & \text{Employment} & \text{Age} & u(x,e_i) & \mathrm{rank}_x(e_i) \\ \hline
e_1 & 28000 & M & \text{Temporary} & 30 & -2 & 1 \\
e_2 & 30000 & F & \text{Permanent} & 35 & -3 & 2 \\
e_3 & 26000 & F & \text{Permanent} & 31 & -3 & 3 \\
e_4 & 35000 & M & \text{Temporary} & 32 & -3 & 4 \\
e_5 & 30000 & M & \text{Permanent} & 35 & -4 & 5 \\
\end{array}
\]

In this set, the utility-optimal counterfactual is \(e_1\), since it changes only two features
(sparsity \(=2\): Income and Gender), hence \(\mathrm{rank}_x(e_1)=1\). \\

However, the explanation provider prefers not to return a counterfactual that changes \textit{Gender},
and therefore selects \(e_4\) instead of the utility-optimal \(e_1\). \\

Thus, by Definition~\ref{def:cherry_picking_non_inject}, \textit{any} explanation other than \(e_1\) is \textit{cherry-picked}.
Since \(e_1\) is top-ranked under the stated sparsity utility, returning any \(e \neq e_1\) implies \(\mathrm{rank}_x(e) > 1\).
\end{example}

In practice, however, an external auditor would not have access to the \textit{real} specification \((\mathcal{F},\mathcal{A},u)\) that governed the provider’s choice. Instead, the auditor only observes a \textit{stated} specification: the models, explanation methods, utility function, and explanation set that the provider discloses, which need not be complete or truthful. The next example shows how this gap allows for \textit{omittance}: by withholding or selectively reporting specification components, a provider can make a cherry-picked explanation appear consistent with the stated specification.

\begin{example}
\textbf{Omittance via post hoc utility selection.}
\label{ex:omittance}

Under Definition~\ref{def:cherry_picking_non_inject}, cherry-picking is determined with respect to the \textit{true} utility \(u\).
However, as auditors only observe \(\EFA{x}\) and the \textit{stated} specefication, and the provider can choose which utility to disclose after inspecting \(\EFA{x}\), cherry-picking can be concealed by disclosing an alternative, seemingly reasonable utility under which the chosen explanation is top-ranked.
We illustrate this using the same \(\EFA{x}\) as in Example~\ref{ex:spars-def}. Assume the provider wishes to report an explanation that does not alter \(\text{Gender}\), and therefore wants to return either $e_2$ or $e_3$.
Under the sparsity utility from Example~\ref{ex:spars-def}, \(e_1\) is top-ranked, so returning $e_2$ or $e_3$ is cherry-picking with respect to utility $u$. \\

To argue that no cherry-picking occurred, the provider discloses a different utility that makes $e_2$ or $e_3$ appear optimal.
One plausible choice is a proximity-based utility defined on normalised feature values, using negative Euclidean distance:
\[
u'(x,e) \;=\; - \lVert x - e \rVert_2
\;=\; -\sqrt{\sum_{i=1}^{d} (x_i - e_i)^2}.
\]

where \(x\) and \(e\) are represented in a normalised and encoded feature space.

Using min--max normalisation for the numerical features (over \(\{x,e_1,\dots,e_5\}\)) and binary encodings
\(F\mapsto 0\), \(M\mapsto 1\), \(\text{Temp}\mapsto 0\), \(\text{Perm}\mapsto 1\),
we obtain:
\[
\begin{array}{c|c|c|c|c|c|c}
\text{ID} & \text{Income} & \text{Gender} & \text{Employment} & \text{Age} & u'(x,e_i) & \mathrm{rank}'_x(e_i) \\ \hline
x   & 0.0 & 0 & 0 & 0.0 & 0.000  & - \\ \hline
e_1 & 0.3 & 1 & 0 & 0.0 & -1.044 & 2 \\
e_2 & 0.5 & 0 & 1 & 1.0 & -1.500 & 4 \\
e_3 & 0.1 & 0 & 1 & 0.2 & -1.025 & 1 \\
e_4 & 1.0 & 1 & 0 & 0.4 & -1.470 & 3 \\
e_5 & 0.5 & 1 & 1 & 1.0 & -1.803 & 5 \\
\end{array}
\]

Crucially, \(u'\) is a standard and seemingly sound choice in its own right (a proximity-based utility can be preferable when sparsity produces many ties in \(\EFA{x}\), as it did in Example \ref{ex:spars-def}), yet the min-max scaling implicitly biases the ranking towards continuous changes over discrete ones, illustrating how omittance can remain plausible and why soundness evaluation is non-trivial. \\

Here \(\mathrm{rank}'_x\) is defined by sorting \(\EFA{x}\) by decreasing \(u'(x,\cdot)\) and setting \(\mathrm{rank}'_x\) equal to the resulting list position. Since \(u'\) assigns distinct values for all explanations in \(\EFA{x}\), i.e., \(\forall\, e_i,e_j \in \EFA{x},\; e_i \neq e_j \Rightarrow u'(x,e_i) \neq u'(x,e_j)\), no tie-breaking rule is required.
Thus, under the disclosed utility \(u'\), the provider can report \(e_3\) and claim that no cherry-picking occurred, since \(\mathrm{rank}'_x(e_3)=1\). \\

The problem is that \(u'\) was selected after inspecting \(\EFA{x}\).
The provider can omit the utility they actually optimised for (the true \(u\) in Definition~\ref{def:cherry_picking_non_inject}) and instead disclose a different, reasonable-looking utility under which the reported explanation becomes top-ranked.
We refer to this as \emph{omittance}.

\end{example}

\section{Detecting Cherry-Picked Explanations}
\label{sec:detecting}

\begin{table}[ht]
  \renewcommand{\arraystretch}{1.3}
  \centering
  \begin{tabular}{p{2.6cm}p{1.5cm}p{1.7cm}p{1.5cm}p{5.5cm}}
    \toprule
    Scenario & Models $\mathcal{F}$ & Methods $\mathcal{A}$ & Utility $u$ & Description \\
    \midrule

    \cellcolor{fullaccess!40}\textbf{Ideal \hspace{2em}(generator)} 
      & \cellcolor{real} Real & \cellcolor{real} Real & \cellcolor{real} Real
      & Ground truth for all components, available only to the generator. \\[0.4em]

    \cellcolor{fullaccess!40}\textbf{Best practice} 
      & \cellcolor{stated} Stated & \cellcolor{stated} Stated & \cellcolor{stated} Stated
      & All components are stated, full code shared and seeds set. \\[0.4em]

    \cellcolor{partialaccess!40}\textbf{Private model} 
      & \cellcolor{estimated} Estimated & \cellcolor{stated} Stated & \cellcolor{stated} Stated
      & Model unavailable; $\mathcal{F}$ estimated from predictions. \\[0.4em]

    \cellcolor{partialaccess!40}\textbf{No model seed} 
      & \cellcolor{estimated} Estimated & \cellcolor{stated} Stated & \cellcolor{stated} Stated
      & Model code shared but seed missing; have to treat $\mathcal{F}$ as containing all possible model seeds \\[0.4em]

    \cellcolor{partialaccess!40}\textbf{No CF seed} 
      & \cellcolor{stated} Stated & \cellcolor{estimated} Estimated & \cellcolor{stated} Stated
      & Method known but runs vary; auditors estimate possible $\mathcal{A}$. \\[0.4em]

    \cellcolor{partialaccess!40}\textbf{No utility seed} 
      & \cellcolor{stated} Stated & \cellcolor{stated} Stated & \cellcolor{estimated} Estimated
      & $u$ not fixed; auditors estimate its variation. \\[0.4em]

    \cellcolor{explainonly!40}\textbf{Explanations only} 
      & \cellcolor{estimated} Estimated & \cellcolor{estimated} Estimated & \cellcolor{estimated} Estimated
      & All components estimated from explanations. \\
    \bottomrule
  \end{tabular}
  \caption{Example scenarios with different levels of access to the models $\mathcal{F}$, counterfactual methods $\mathcal{A}$, and utility $u$. 
  \textcolor{real}{\textbf{Green}} = Real, 
  \textcolor{stated}{\textbf{Yellow}} = Stated, 
  \textcolor{estimated}{\textbf{Red}} = Estimated. 
  Only the ideal case contains real components. Row shading in the \emph{Scenario} column indicates access level, with \textcolor{fullaccess}{dark blue} for Full Procedural Access, \textcolor{partialaccess}{medium blue} for Partial Procedural Access, and \textcolor{explainonly}{light blue} for Explanation-Only Access. When \(u\) is non-injective, determining whether an explanation is cherry-picked also requires the (often implicit) tie-breaking rule, equivalently the induced ranking \(\mathrm{rank}_x\).}
  \label{tab:component_realism}
\end{table}

In order to determine whether an explanation is cherry-picked according to Definition~\ref{def:cherry_picking_non_inject},
one would need the \emph{real} explanation space \(\EFA{x}\), the utility function \(u\), and the induced ranking procedure \(\mathrm{rank}_x\).
In practice, auditors only observe a \emph{stated} specification, and can at best verify whether the reported explanation is consistent with it,
that is, whether it could be produced by the stated \((\mathcal{F},\mathcal{A})\) and whether it is top-ranked under the stated \(u\) (and its tie-breaking rule).
This verification does not rule out omittance, since undisclosed alternatives may have been considered and withheld.
In this section, we discuss three levels of access to the explanation process: full procedural access, partial procedural access, and explanation-only access.
Table~\ref{tab:component_realism} highlights example scenarios across these settings.
\footnote{Code to reproduce the empirical results in this section is available at \url{https://github.com/JamesHinns/cherry_picking_counterfactuals}}.
\vspace{-1em}

\subsection{Full Procedural Access}

We start with the case where an external auditor has full access to the stated explanation specification, given by $\mathcal{M}$, $\mathcal{F}$, $u$. 
Throughout this section, when we refer to access to the stated utility \(u\), we assume that the corresponding deterministic tie-breaking convention is also stated, so that the induced ranking \(\mathrm{rank}_x\) is well-defined.
In Table~\ref{tab:component_realism}, this is \textit{Best practice}; the \textit{Ideal} remains out of reach, since a provider can omit undisclosed alternatives they evaluated. As a result, the auditor can check consistency with the stated $\mathcal{M}$, $\mathcal{F}$, and $u$, but cannot establish that these are the full set of components considered.
In this scenario, the generator provides the full code used to generate the explanations, so we have a stated set of models $\mathcal{F}$ and explanation methods $\mathcal{A}$, as well as the utility function $u$.

\begin{example}
    Consider the following scenario:
    Bob provides counterfactual explanations to an auditor Alice.
    Alice requests Bob states his full procedure to select counterfactuals for the first 10  instances for the test set on the Adult dataset $\mathcal{X}_e$. Bob uses a random forest model with a fixed seed, so $\mathcal{F}$ only contains a single model $\mathcal{F} = \{f\}$. He then uses DiCE random with 10 different random seeds, meaning $\mathcal{A}$ contains 10 different methods, and as such at each instance $x \in \mathcal{X}_e$ there are ten explanations in the explanation space $\EFA{x}$. 
    He then states he uses a utility function based on proximity, defined as the negative of the euclidean distance between the factual and counterfactual:
    \[
        u(x,e) \;=\; - \lVert x - e \rVert_2,
    \]

    Alice attempts to recreate this but finds that the explanations Bob returns are not always the optimal in $\EFA{x}$ according to the stated $u$ and as such are cherry-picking.

    This is because in reality Bob wanted to ensure the features: sex, relationship and martial status were never present in the returned explanations, so actually used the utility function:

   \[
        S \;=\; \{\text{sex},\ \text{relationship},\ \text{marital}\}.
        \]
        \[
        u'(x,e) \;=\;
        \begin{cases}
        100000, & \text{if } \exists\, s\in S:\ s(x)\neq s(e),\\[4pt]
        -\lVert x-e\rVert_2, & \text{otherwise.}
        \end{cases}
    \]
    We show a summary of this in Figure \ref{fig:cherry-pick_example_prox}.
    In this scenario, proximity is sufficiently specified to separate the optimal from the cherry-picked, the cherry-picked explanations are at most equal in the reported utility and averagely worse.

    \begin{center}
      \includegraphics[width=0.8\linewidth]{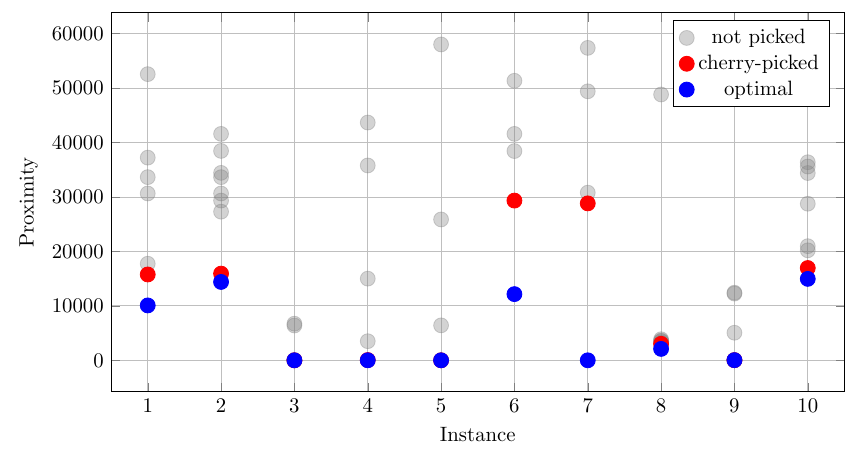}
      \captionof{figure}{Proximity of optimal vs cherry-picked explanation space for DiCE random with 10 different seeds. \vspace{-1em}}
      \label{fig:cherry-pick_example_prox}
    \end{center}
    
\label{ex:full_cherry_pick_detect}
\end{example}

This example shows that, under full procedural access, a sufficiently discriminative stated utility can make cherry-picking verifiable by checking whether the reported explanation is top-ranked in \(\EFA{x}\).
The next example shows a key limitation: even when the reported explanations appear consistent with the disclosed utility on a finite audit set, a provider may still have selected them under a different utility and then disclosed an alternative utility that looks legitimate and is observationally equivalent on the audited instances.

\begin{example}
Consider the same setting as example \ref{ex:full_cherry_pick_detect}, but instead of a proximity-based utility, Bob reports a sparsity based metric.
\[
u(x,e) \;=\; -\lVert x-e\rVert_0,
\]
He once more wants to not return any counterfactuals changing a restricted set of features $S=\{\text{sex},\text{relationship},\text{marital}\}$, relationship or martial status and so actually uses a utility function $u'$. We also let $d$ be the number of dimensions of the data $\mathcal{X}$ and explanation $\mathcal{E}$ space (the number of features in the data): $x,e \in \mathbb{R}^d$.
\[
u'(x,e)
\;=\;
-\min\!\left\{
\lVert x-e\rVert_0 \;+\; d\cdot \mathbf{1}\!\left[\exists\, s\in S:\ s(x)\neq s(e)\right],
\ d
\right\}.
\]

If Alice evaluates the reported explanations under \(u\), then on the first ten instances in \(\mathcal{X}_e\) she observes no discrepancy: for each \(x\), the explanation returned by optimising \(u'\) attains the same \(u(x,\cdot)\) value as an explanation that could have been selected by optimising \(u\). Consequently, based on these ten instances alone, the disclosed utility \(u\) does not allow Alice to certify that Bob avoided restricted-feature changes by switching to \(u'\), since the observed outputs remain compatible with the stated specification. However, this apparent agreement is an artefact of the limited audit set: if Alice instead evaluates the first 100 test instances, then for 8 instances the explanation selected by \(u'\) is strictly worse under \(u\) than the \(u\)-optimal alternative in \(\EFA{x}\), which would constitute detectable cherry-picking with respect to \(u\).

\label{ex:full_cherry_pick_nondetect}
\end{example}

These examples are representations of real runs on the Adult dataset. They instantiate the same underlying explanation set \(\EFA{x}\), but differ in the stated specification used to justify the reported explanation. This illustrates omittance as in Example~\ref{ex:omittance}: by withholding or selectively reporting components of the true specification \((\mathcal{F},\mathcal{A},u)\), a provider can make a cherry-picked explanation appear consistent with what is stated to an auditor. \\

Here, the key degrees of freedom are edits to the counterfactual method \(\mathcal{A}\) (which determine which feature changes are feasible) and the disclosed utility \(u\) (which determines the ranking over \(\EFA{x}\)). Since DiCE random only weakly optimises for sparsity, the number of feature changes is often a poor discriminator between these settings, so restricted-\(\mathcal{A}\) and unrestricted-\(\mathcal{A}\) runs can yield counterfactuals with similar sparsity despite different semantics. More generally, concealment becomes easier when the specification induces many ties (or near-ties) among explanations. Omittance while presenting a sound-looking specification is easier when many reasonable choices exist, since limited degrees of freedom make it harder to swap components without affecting behaviour.

This also shows that, for an individual data instance, it can be simple to disclose a seemingly reasonable \(u\) under which a preferred explanation is top-ranked, even if \(u\) was selected after inspecting \(\EFA{x}\). In later experiments we therefore move to mean utilities across multiple instances, reflecting the more complex and less informative auditing settings considered in the remainder of the section; however, omittance remains possible whenever evaluation is based on a finite set of instances.

\subsection{Partial Procedural Access}

In this setting, the auditor has access to some, but not all, information about the explanation pipeline. As shown in Table \ref{tab:component_realism}, partial access covers a wide range of scenarios. 
For example, the code may be available but the random seeds for counterfactual generation are not fixed; alternatively, the model may be private, so the auditor must rely on query access or on a surrogate that provides predictions for a given instance.

Based on Definition \ref{def:cherry_picking_non_inject}, an ideal detector for cherry-picking would construct the explanation space $\EFA{x}$ and check whether there exists an explanation $e' \in \EFA{x}$ whose utility exceeds that of the reported explanation $e$. The explanation space $\EFA{x}$ is defined in Definition \ref{def:cherry_picking_non_inject} as
\[
\EFA{x}
  \;=\;
  \{\, A_f(x) \;\mid\; f \in \mathcal{F},\ A \in \mathcal{A} \,\}.
\]

If a counterfactual method is non-deterministic, each execution must be treated as a distinct admissible mapping, since every $A \in \mathcal{A}$ is required to be single-valued. Consequently, a single non-deterministic procedure can induce infinitely many possible mappings $A$, corresponding to different realisations of its randomness. Although practical implementations typically impose finite constraints, such as finite seeds or datasets, explicitly enumerating the explanation space $\EFA{x}$ can still be computationally infeasible.

\begin{figure}
  \centering
  \includegraphics{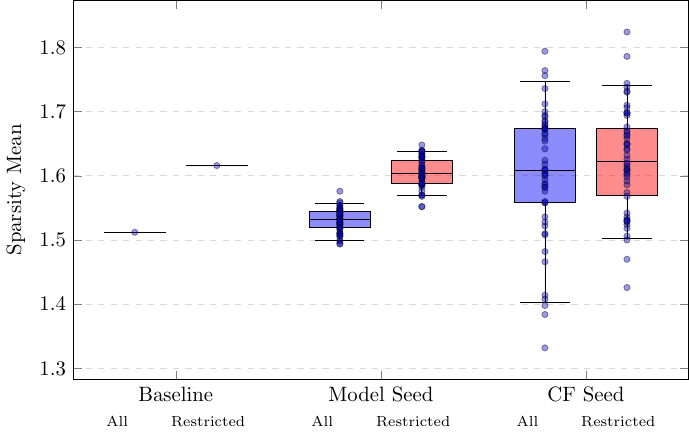}
  \caption{Mean sparsity of explanations generated by DiCE (random), comparing a baseline run with 50 different seeds for both the model and the counterfactual generation. Results are shown for \textbf{All}, where all features may be edited, and \textbf{Restricted}, where edits are limited to a predefined restricted feature set (non-sensitive features only).}
  \label{fig:sparsity_dice}
\end{figure}

Verifying the complete set $\EFA{x}$ by repeatedly calling the non-deterministic explanation algorithm would amount to a halting problem in a continuous space, since it would always be possible for a new run of $A \in \mathcal{A}$ to produce a previously unseen value. Informally, the halting problem is the fact that there is no general procedure that can decide, for every possible algorithm and input, whether the algorithm will ever terminate (or, equivalently here, whether continuing to run it could still yield something new). However, as we will discuss in Remark \ref{ex:restriction_detectable}, in a discretised setting the space is not truly infinite, and so it may be possible, in principle, to produce every possible discrete value in the space (although this may be unlikely in practice). Therefore, without any explicit bounding of $\mathcal{F}$ or $\mathcal{A}$, defining an explanation as cherry-picking is infeasible.

\begin{figure}
    \centering
    \includegraphics[width=0.8\linewidth]{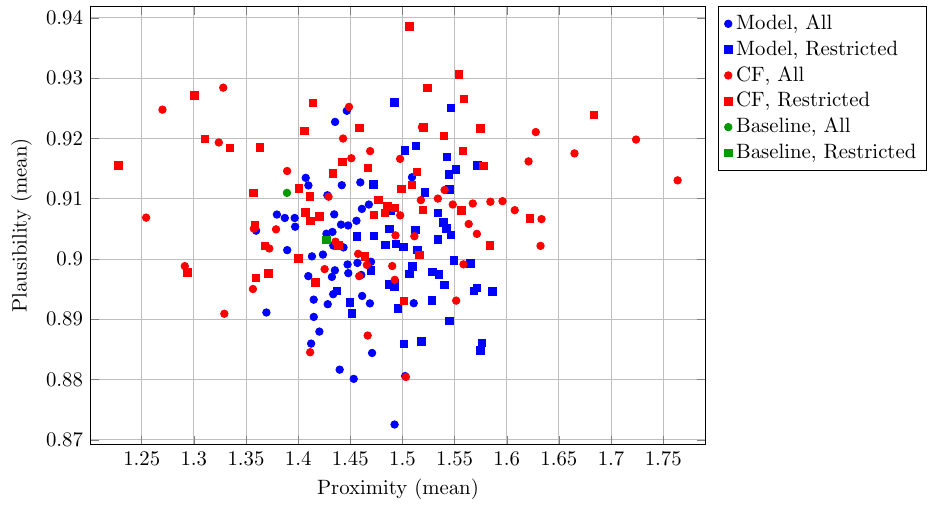}
    \caption{Mean plausibility versus proximity on the Adult dataset (50 seeds per group). Colours indicate the source of variation: Baseline uses a fixed seed (42), while Model and CF represent variations where the model seed or counterfactual initialisation seed are randomised, respectively. The second term denotes the feature constraints: \textit{All} permits edits to any feature, whereas \textit{Restricted} limits edits to non-sensitive features only.}
    \label{fig:prox_plaus_50_repeats}
\end{figure}

In Examples \ref{ex:full_cherry_pick_detect} and \ref{ex:full_cherry_pick_nondetect}, we showed how different utility functions can support different narratives. We now show empirically that randomness in the explanation procedure can induce variation that matches, and sometimes exceeds, the variation induced by explicit design choices such as restricting which features may be edited. This matters because it can be difficult to distinguish intentional selection from variation that could plausibly arise under unrestricted seed choices. In Figures \ref{fig:prox_plaus_50_repeats} and \ref{fig:sparsity_dice}, we generate counterfactual explanations with DiCE for a Random Forest model, comparing an unrestricted setting (all features editable) to a restricted setting in which only non-sensitive features may be edited, mirroring the cherry-picking done in Examples \ref{ex:full_cherry_pick_detect} and \ref{ex:full_cherry_pick_nondetect}. \\

DiCE random does not optimise a strict utility function (when total cfs is set to 1) when generating single counterfactuals; rather, it weakly optimises for sparsity. It does so by randomly swapping $n$ features with their respective values from other points in the training set, starting with $n=1$ and increasing $n$ until a valid counterfactual is found. Since the method returns the first valid solution it encounters, both stochasticity and the admissible edit space can substantially affect the sparsity of the returned explanation. \\

In Figure \ref{fig:sparsity_dice} we show that the effect of restricting editable features changes average sparsity relative to editing all features when varying seeds. We run 50 random seeds for the model and 50 random seeds for counterfactual generation, and repeat each sweep under both feature-edit settings, yielding 200 runs in total. Each point reports mean sparsity over the same 500 test instances. Pairing restricted and unrestricted runs by seed choice gives 100 comparisons. In 25 of these 100 pairs, the restricted setting attains lower mean sparsity than its unrestricted counterpart. This is not evidence that restriction is inherently beneficial; rather, because DiCE terminates as soon as it finds a valid counterfactual, restricting the search can steer the heuristic away from non-meaningful changes, even while other potentially better explanations remain unexplored. \\

Figure \ref{fig:prox_plaus_50_repeats} shows a similar sensitivity for proximity (HEOM) and plausibility (IM1), and also highlights the added difficulty when the auditor does not know which utility, if any, the provider prioritised. This is also shown in Figure \ref{fig:partial_runs:seed_method} with only 10 repeats, to clearly see the diffences between the baseline, the cherry-picked (restricted feature set, same seeds as baseline), and different seeds for model and counterfactual generation. \\

Figure \ref{fig:sparsity_dice} further shows that the variation induced by changing the model seed is smaller than the variation induced by the counterfactual generation procedure, aligning with the findings of \cite{brughmans2024disagreement}. Finally, Figure \ref{fig:partial_runs:method} shows that variation between DiCE methods outweighs the variation within DiCE random. This level of between-method variation creates additional scope for cherry-picking across counterfactual methods, even when the methods themselves are deterministic. \\

Together, these results show that even when code is shared and reported results are reproducible, the risk of cherry picking persists: many alternative runs, settings, or methods may have been tested and then left out of the stated specification through omittance.

\begin{figure}[t]
  \centering
  \begin{subfigure}[t]{0.48\textwidth}
    \centering
    \includegraphics[height=7cm, keepaspectratio]{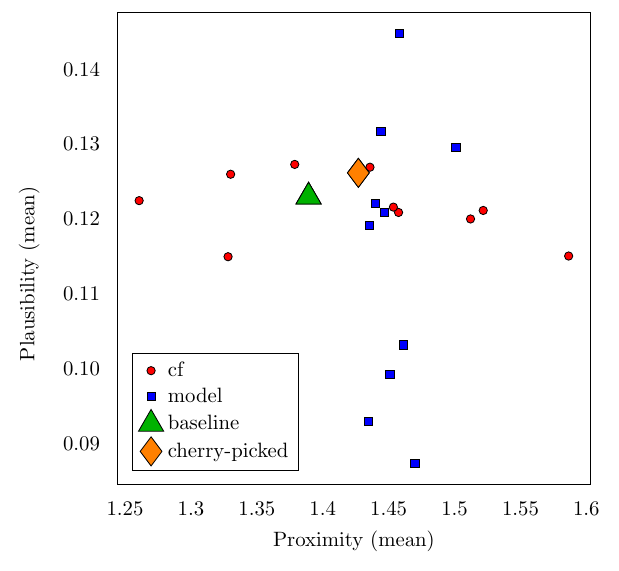}
    \caption{Effect of varying the model and counterfactual seeds on mean plausibility
    and proximity for the Adult dataset, compared with a cherry-picked configuration
    that restricts edits to sensitive features.}
    \label{fig:partial_runs:seed_method}
  \end{subfigure}\hfill
  \begin{subfigure}[t]{0.48\textwidth}
    \centering
    \includegraphics[height=7cm, keepaspectratio]{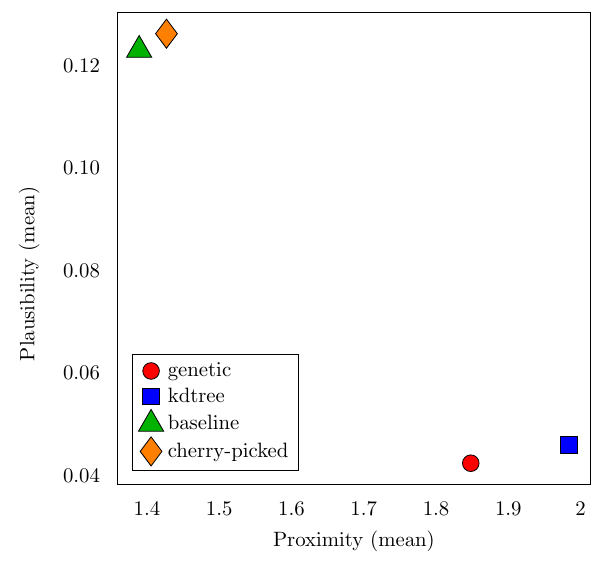}
    \caption{Comparison of dice-random cherry-picked and baseline against other dice methods.}
    \label{fig:partial_runs:method}
  \end{subfigure}

  \caption{Mean Proximity and Plausibility of cherry-picked vs.\ alternative explanation
  settings for the Adult dataset. \vspace{-1em}}
  \label{fig:partial_runs}
\end{figure}

\subsection{Explanation-Only Access}

In the explanation-only setting, the most restrictive access level, we illustrate the scale of the detection problem.
When auditors or explanation recipients have access only to the counterfactuals themselves, and the model owner provides no information about the model or the explanation method, it becomes very difficult to characterise the admissible explanation space \(\EFA{x}\) or the utility function \(u\). Such a situation may arise when several recipients of explanations from an explanation provider that shares no procedural details decide to work together to check for possible manipulation.
In practice, this may happen when recipients pool their own instances and explanations, either informally or via journalistic investigations. Independently, an auditor might only obtain the reported explanations, without access to the model or the explanation procedure.
Without insight into how the explanations were produced, it is difficult to determine which alternatives were feasible or whether the presented explanations were selectively chosen. Although cherry picking detection is not logically impossible at this level, it is infeasible in practice under almost all realistic parameter settings.

\subsubsection{Undetectability with no restrictions on model or counterfactual generation.}

\begin{proposition}
\textbf{Cherry-picking is undetectable without restrictions on $\mathcal{F}$, $\mathcal{A}$ and $u$}
\label{prop:undetectable} \\
Fix $x \in \mathcal{X}$. Suppose that both the model class $\mathcal{F}$ and the set of explanation methods $\mathcal{A}$ are unrestricted, meaning every possible model $f:\mathcal{X}\to\mathcal{Y}$ and every possible mapping $A:\mathcal{F}\times\mathcal{X}\to\mathcal{E}$ are admissible. Then $\EFA{x} = \mathcal{E}$, and it is impossible to determine whether a reported explanation $e \in \EFA{x}$ is cherry-picked under Definition~\ref{def:cherry_picking_non_inject}.
\end{proposition}

\begin{proof}
By definition,
\[
\EFA{x} = \{\,A_f(x) \mid f \in \mathcal{F},\, A \in \mathcal{A}\,\} \subseteq \mathcal{E}.
\]
If $\mathcal{F}$ and $\mathcal{A}$ are unrestricted, then for any $e \in \mathcal{E}$
there some $f \in \mathcal{F}$ and $A \in \mathcal{A}$ such that $A_f(x) = e$.
Hence $\EFA{x} = \mathcal{E}$. By Definition~\ref{def:cherry_picking_non_inject}, an explanation $e$ is not cherry-picked
if and only if $\mathrm{rank}_x(e) = 1$, that is, if no admissible explanation is ranked
strictly above $e$.
Verifying this condition requires ruling out the existence of any
$e' \in \EFA{x}$ with $\mathrm{rank}_x(e') < \mathrm{rank}_x(e)$.

When $\EFA{x} = \mathcal{E}$ and no restrictions are imposed on
$\mathcal{F}$, $\mathcal{A}$, or $u$, there is in general no finite or effective procedure
to establish that such an $e'$ does not exist.
Therefore, without restrictions on $\mathcal{F}$, $\mathcal{A}$, or $u$,
cherry-picking is undetectable.
\end{proof}

\begin{remark}
\label{ex:restriction_detectable}
In practice, $\mathcal{F}$ and $\mathcal{A}$ are not enumerated explicitly and may be very large
(or even infinite) as sets of procedures or parameterisations.
However, for any fixed instance $x$, implementations operate under finite numerical precision
and return a finite-precision explanation representation.
Consequently, although there may be infinitely many admissible pairs $(A,f)$, the set of
distinct outputs in $\EFA{x}=\{A_f(x)\mid f\in\mathcal F, A\in\mathcal A\}$ is at most
countable (and typically finite).

Consequently, $\EFA{x}$ is finite or countable but typically enormous. In the worst case, certifying that a reported explanation $e$ is not cherry-picked (i.e., is top-ranked in $\EFA{x}$) requires evaluating $u(x,e')$ for all $e'\in\EFA{x}$, which is $\Theta(|\EFA{x}|)$ time. Conversely, cherry-picking can be detected in $\Theta(1)$ time by exhibiting a single $e'\in\EFA{x}$ with $u(x,e')>u(x,e)$, but still has worst-case $\Theta(|\EFA{x}|)$ time when no such witness exists. Hence, while detection is not logically impossible, ruling out cherry-picking is generally infeasible for large $\EFA{x}$.
\end{remark}

\subsubsection{Restricted Model Class}
\label{subsubsec:restrictions}

This subsection discusses a side case where some explanation behaviours can be shown to be impossible under the stated model set $\mathcal{F}$, even when $\mathcal{F}$ and the explanation method $A$ are not fully specified. Below, we show how restricting the complexity of models in $\mathcal{F}$ via a bound on $d_{\mathrm{VC}}(\mathcal{F})$ can rule out such behaviours and thereby expose cherry-picking. Here, $d_{\mathrm{VC}}(\mathcal{F})$ denotes the Vapnik--Chervonenkis (VC) dimension, a measure of the expressive capacity of a model family: informally, it is the largest number $n$ for which there exists a set of $n$ input points such that, for every assignment of binary labels to those points, some model in $\mathcal{F}$ realises that assignment. We use VC dimension as a coarse capacity proxy: if the decision-boundary behaviour implied by a report exceeds this capacity, it cannot be realised by any $f \in \mathcal{F}$. \\

Bounding $d_{\mathrm{VC}}(\mathcal{F})$ limits how complex the decision boundaries in $\mathcal{F}$ can be, and thus constrains which explanation behaviours are attainable under any $f \in \mathcal{F}$. In practical terms, such a restriction tends to correspond to simpler model families (for example linear classifiers or shallow decision rules) rather than highly flexible families such as large neural networks. \\

Some sectors, for example retail banking where linear scorecards have long been standard, often default to comparatively simple model classes for reasons of robustness, interpretability, and ease of governance. However, these practices are largely institutional rather than legally mandated, and therefore do not guarantee any enforceable bound on model complexity.
Broader frameworks such as the GDPR and the EU AI Act emphasise transparency and push towards explainability in algorithmic decision-making~\cite{ponce2025right,goodman2017european}, but they do not constrain the underlying model class.
Recent work also notes that post-hoc explanations are frequently treated in policy debates as
the primary means of meeting these transparency obligations, rather than limiting the complexity
of the predictive model~\cite{bordt2022post}.

\begin{proposition}
\textbf{Detectability of omittance under restricted model classes}
\label{prop:restricted}\\
Let $\mathcal{F}$ be a set of functions with finite VC dimension
$d_{\mathrm{VC}}(\mathcal{F})<\infty$.
Let $E\subseteq\mathcal E$ be a reported set of counterfactual explanations.
If the decision-boundary behaviour collectively implied by $E$ would require a set of functions
with VC dimension strictly greater than $d_{\mathrm{VC}}(\mathcal{F})$, then there exists $x\in\mathcal X$
such that $E \nsubseteq \EFA{x}$.
Consequently, the report is inconsistent with the stated specification and indicates omittance: at least one reported explanation cannot be obtained from any admissible pair $(f,A)$ with $f\in\mathcal F$ and $A\in\mathcal A$ at the corresponding instance.
\end{proposition}

\begin{proof}
If $\mathcal{F}$ has finite VC dimension, then the set of functions it contains has bounded expressive capacity, and in particular it cannot realise arbitrarily complex label patterns on finite subsets of $\mathcal{X}$.
Any explanation method $A\in\mathcal A$ whose output depends on $f$ is therefore constrained by the representational limits of $\mathcal F$. Suppose a reported explanation set $E$ implies decision-boundary behaviour that would require VC dimension strictly greater than $d_{\mathrm{VC}}(\mathcal F)$. Then there exists no $f\in\mathcal F$ for which the reported behaviours are simultaneously attainable, and hence there exists an instance $x\in\mathcal X$ such that $E \nsubseteq \EFA{x}$ (equivalently, at least one reported explanation in $E$ is not contained in $\EFA{x}$). Consequently, if such an $E$ is reported as originating from $\mathcal F$, the report is inconsistent with the stated model class: it indicates either use of a more expressive family than $\mathcal F$, a different explanation mechanism than assumed, or an additional generative process outside the scope of $\mathcal F$.
\end{proof}

\begin{example}
\textbf{Explanations too complex for a simple model}
\label{ex:too_complex} \\

Consider a simple instance space $\mathcal{X} = \{\text{Income}, \text{Gender}\}$, where $\mathrm{Income} \in [0,100000] \cap \mathbb{N}$ and $\mathrm{Gender} \in \{0,1\}$. 
Define a manual rule-based function 
\[
  f_m = \mathrm{Gender} \oplus (\mathrm{Income} > 50{,}000),
\]
that is, the exclusive-or between gender and income being greater than 50000. 
This relationship is non-linear and cannot be represented by any linear model. 
A logistic regression trained on such data achieves approximately 50\% accuracy (in a balanced dataset like our example), predicting only the majority class, since the true decision function alternates between classes in a non-linearly separable manner.

\captionsetup{type=figure}
\includegraphics[width=\textwidth]{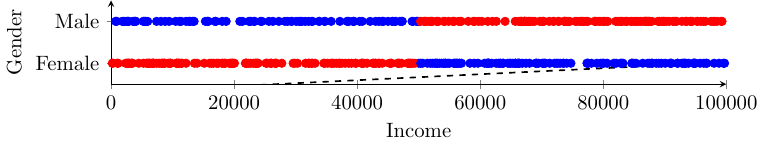}
\captionof{figure}{Example of an explanation pattern too complex for a linear model to capture. 
Colour indicates the true class label, while the dotted line shows the decision boundary of a logistic regression model.}
\label{fig:too_complex}

If a set of explanations $E$ derived from $f_m$ were claimed to come from a linear model, this would contradict the representational capacity of $\mathcal{F}$, as $E$ encodes non-linear dependencies beyond its VC dimension. Hence, the report would be inconsistent with the stated specification and indicates omittance: there exists $x\in\mathcal X$ such that $E \nsubseteq \EFA{x}$, meaning that at least one reported explanation cannot be obtained from any admissible pair $(f,A)$ with $f\in\mathcal F$ and $A\in\mathcal A$ at the corresponding instance. This illustrates how restricting $\mathcal{F}$ can make omittance theoretically detectable.
\end{example}

\section{Discussion}
\label{sec:discussion}

This paper studies cherry-picking in counterfactual explanations and the extent to which it can be detected by external auditors. We formalise cherry-picking relative to an admissible explanation space and a utility-based ranking.
We evaluate detection across three information access settings. In a \textit{full procedural access} setting, code, data, model(s), parameters, and seeds are stated, so an auditor can test conformity to the declared procedure. In a \textit{partial procedural access} setting, only some details are stated, so partial consistency checks are possible. In an \textit{explanation-only access} setting, only reported explanations are visible, so verification is rarely possible. The strength of any conclusion follows the information access setting, but our results show that detection is severely limited in practice.
\vspace{-1em}
\paragraph{The hidden degrees of freedom:}
The core issue arises from the combination of two features: the multiplicity of admissible explanation specifications and incomplete disclosure of how these procedures are instantiated. Because many $(f,A)$ pairs are formally permissible, and because the sets $\mathcal{F}$ and $\mathcal{A}$ are typically only partially specified, explanation pipelines contain latent degrees of freedoms that are invisible to external auditors.

These hidden degrees of freedom enable selective reporting. When the explanation algorithm is non-deterministic, a provider can repeatedly run the same nominal procedure, varying random seeds or stochastic optimisation trajectories, and suppress runs that exhibit undesirableso behaviour. Importantly, this vulnerability is not restricted to stochasticity alone. Even under seemingly deterministic reporting (e.g., reporting fixed seeds), selective reporting can occur if multiple seeds, hyperparameter configurations, initialisations, or compute budgets were explored but only a favourable subset is disclosed. 

This problem is structural rather than incidental. The disagreement problem implies that for any given instance, there exists a plurality of equally valid explanations~\cite{brughmans2024disagreement, goethals2023manipulation}. At every upstream stage (such as explanation algorithm, utility specification, feasibility constraints, data selection) multiple reasonible design choices exist. Together, these choices generate a large space of plausible explanations, from which a provider can select ex post, while portraying the final output as the result of a single, fully specified procedure.

\paragraph*{Why metric-based detection fails in practice:}
Metric-based detection presumes that cherry-picking induces a detectable shift in counterfactual quality scores. In practice, this assumption fails because the explanation generation process itself exhibits substantial inherent variability.

Most counterfactual methods rely on local or heuristic optimisation for tractability rather than exhaustive search. As a result, each run explores only a small and contingent portion of an extremely large admissible explanation space. This alone introduces considerable variation across runs.
Our experiments show that detection based on traditional counterfactual quality metrics fails consistently. Variation across seeds often exceeds the effect of the cherry-picking intervention itself. As a result, cherry-picked explanations fall well within the distribution of non–cherry-picked runs and frequently lie near its centre.\footnote{Importantly, cherry-picking does not necessarily degrade counterfactual quality metrics and can even improve them. For example, a provider may restrict the search space by disallowing edits to sensitive attributes such as gender. Under a fixed compute budget, this reduced space is easier to optimise, allowing the optimiser to return higher-scoring explanations while systematically excluding behaviours an auditor would wish to observe.}
Consequently, selectively reporting favourable runs while omitting unfavourable ones can remain statistically inconspicuous to metric-based detection

\paragraph*{From soundness to specification verification:}

Because outcome-based metrics cannot reliably expose cherry-picking, we shift attention from soundness to specification verification: Do the reported explanations, in fact, follow the stated specification?

This question is only meaningfully addressable under full access, where an auditor can rerun the declared procedure, verify determinism claims, and probe sensitivity to seeds and implementation details. This does not answer the question about the appropriateness of the specification itself, which will depend on the domain and governance assessment.

Crucially, even perfect reproducibility of a \textit{stated} run does not rule out omittance. Reproducibility can establish that a reported explanation could have been produced by the stated procedure, but not that it is representative of all runs that were performed. Verification can therefore confirm conformance to a declared specification, but it cannot establish completeness of disclosure.

\paragraph*{Implications for mitigation and auditing practice:}

Mitigation efforts should therefore prioritise reducing hidden degrees of freedom, rather than attempting to detect cherry-picking post hoc through metrics.

First, explanation specifications should be explicit and technically reproducible. This includes full disclosure of parameter settings, all relevant random seeds, and the computational budget used to generate explanations. Where constraints depend on compute, budgets should be expressed in evaluations or iterations rather than time; if time limits are unavoidable, they should be reported together with the hardware.

Second, when stochasticity is unavoidable, seed choice should be constrained by policy rather than left to provider discretion, for example, by requiring pre-declared seed sets or fixed evaluation protocols that aggregate over multiple seeds. A central concern is representativeness: do the reported explanations reflect typical behaviour across admissible runs, or could they plausibly be atypical and favourable? This question cannot be reliably answered under explanation-only access and is only partially addressable under partial access.

These considerations motivate standardisation as a governance response. Since appropriate specifications depend on domain-specific norms, risks, and regulatory objectives, and lack a unique ground-truth solution, domain-level convergence is especially important. In high-stakes domains such as banking, shared standards covering utility functions, feasibility constraints, model classes, counterfactual methods, and reporting protocols can substantially reduce the scope for selective reporting. 

Finally, we offer concise role-specific guidance to mitigate the risk of cherry-picking:
\begin{itemize}
    \item Algorithm developers should expose full seed control and minimise hidden sources of non-determinism.
    \item Explanation providers should prioritise strict reproducibility and disclose the complete explanation specification, including all relevant parameter settings and compute budgets.
    \item Auditors should, where possible, reproduce declared runs and stress-test sensitivity to seeds and specification components, recognising that such checks verify conformance to what is declared, not the absence of omittance.
    \item Regulators and governors should converge on domain-specific specifications that explanation providers are expected to follow.
\end{itemize}

\vspace{-1em}
\section{Conclusion}

Cherry-picking in explanations is a real and present risk, and detecting it is genuinely hard. Our results show that even under strong transparency assumptions, cherry-picking can remain statistically and procedurally indistinguishable from ordinary variability. This difficulty is not accidental, but follows from how counterfactual explanations are generated and selected in practice. Even reproducibility of a stated procedure does not certify that it was faithfully followed or that omitted alternatives were not considered. 

Without effective safeguards, explainability can therefore be used to obscure rather than to clarify. When this happens, explanations may appear reasonable and well-behaved, while still hiding important aspects of model behaviour. As a result, post hoc detection methods and explanation quality metrics provide only limited protection, especially when access to the explanation process is partial or restricted. In such settings, cherry-picking can remain indistinguishable from ordinary variation in explanation outcomes. This makes clear that \textit{post hoc} detection is an insufficient governance strategy on its own, and that stronger, domain-specific \textit{ex ante} constraints are needed on how explanations are specified, generated, and reported, rather than relying on outcome-based checks after the fact.

\vspace{-1.1em}
\section*{Acknowledgements}
This research received funding from the Flemish AI Research (FAIR) Program, and the Flemish Research Foundation (FWO grants 1247125N and G0G2721N).

\bibliographystyle{plain}
\bibliography{ref}

\end{document}